\documentclass{article}

\usepackage{arxiv}
\usepackage{enumitem}
\usepackage{amsmath}
\usepackage{amsthm}

\usepackage[utf8]{inputenc} 
\usepackage[T1]{fontenc}    
\usepackage{hyperref}       
\usepackage{url}            
\usepackage{booktabs}       
\usepackage{amsfonts}       
\usepackage{nicefrac}       
\usepackage{microtype}      
\usepackage{lipsum}		
\usepackage{graphicx}
\usepackage{natbib}
\usepackage{doi}

\title{Budgeted Combinatorial Multi-Armed Bandits}

\author{ Debojit Das \\
	International Institute of\\ Information Technology (IIIT)\\
	Hyderabad, India \\
	\texttt{debojit.das@research.iiit.ac.in} \\
	\And
	Shweta Jain \\
	Indian Institute of Technology (IIT)\\
	Ropar, India \\
	\texttt{shwetajain@iitrpr.ac.in} \\
	\And
	Sujit Gujar \\
	International Institute of\\ Information Technology (IIIT)\\
	Hyderabad, India \\
	\texttt{sujit.gujar@iiit.ac.in} \\
}

\renewcommand{\headeright}{}
\renewcommand{\undertitle}{}
\renewcommand{\shorttitle}{Budgeted Combinatorial Multi-Armed Bandits}

\hypersetup{
pdftitle={Budgeted Combinatorial Multi-Armed Bandits},
pdfsubject={},
pdfauthor={Debojit Das, Shweta Jain, Sujit Gujar},
pdfkeywords={Combinatorial MAB, Budgeted MAB, Crowdsourcing, Heterogeneous Arms},
}

\usepackage{booktabs}
\usepackage{xcolor}
\usepackage{graphicx}
\graphicspath{ {./Images/} }
\usepackage{subcaption}
\usepackage[ruled,vlined]{algorithm2e}
\usepackage{multicol}
\usepackage{tabularx}
\usepackage{algorithmic}
\usepackage{soul}
\usepackage{natbib}

\newtheorem{theorem}{Theorem}
\newtheorem{Theorem}[theorem]{Theorem}
\newtheorem{Claim}[theorem]{Claim}
\newtheorem{Lemma}[theorem]{Lemma}
\newtheorem{Proposition}[theorem]{Proposition}

\newcommand{\CBwK}{\textsf{CBwK-LP}}
\newcommand{\ourbwk}{\textsf{CBwK-LP-UCB}}
\newcommand{\greedy}{\textsf{CBwK-Greedy-UCB}}
\newcommand{\offgreedy}{\textsf{CBwK-Greedy}}

\begin{document}
\maketitle

\begin{abstract}
	We consider a budgeted combinatorial multi-armed bandit setting where, in every round, the algorithm selects a \emph{super-arm} consisting of one or more arms. The goal is to minimize the total expected regret after all rounds within a limited budget. Existing techniques in this literature either fix the budget per round or fix the number of arms pulled in each round. Our setting is more general where based on the remaining budget and remaining number of rounds, the algorithm can decide how many arms to be pulled in each round.  First, we propose \greedy\  algorithm, which uses a greedy technique, \offgreedy, to allocate the arms to the rounds. Next, we propose a reduction of this problem to Bandits with Knapsacks (BwK) with a single pull. With this reduction, we propose \ourbwk\ that uses PrimalDualBwK ingeniously. We rigorously prove regret bounds for \ourbwk. 
    We experimentally compare the two algorithms and observe that \greedy\  performs incrementally better than \ourbwk. We also show that for very high budgets, the regret goes to zero.
\end{abstract}


\maketitle


\section{Introduction}
\label{sec:intro}

    \noindent In a classical stochastic Multi-Armed Bandit (MAB), there are a set of available arms (bandits), and each arm, when pulled, generates a reward from an unknown distribution. A typical goal is to design an algorithm to optimally pull one arm at each round to learn the mean of their reward distributions while maximizing their earned reward. The underlying concept behind multi-armed bandit problems is the trade-off between \emph{exploration} and \emph{exploitation}. The dilemma is as follows: exploring an arm may lead to playing suboptimal arms, whereas simply exploiting the arm with the highest reward may lead to wrongly identifying some arm as the optimal arm. Therefore, striking a balance is crucial. There is a vast amount of literature for a classical stochastic multi-armed bandit, and optimal algorithms exist to solve the problem up to a constant factor. Among the frequentist approaches, UCB1 algorithm by \citet{Auer2002} achieves a regret matching with the lower bound regret for classical bandit setting \cite{LaiRobbins1985}. Whereas in Bayesian approaches,  \citet{thompson} proposed a Thomson Sampling algorithm which obtains the optimal regret bounds up to a constant factor \cite{agrawal12,kaufmann12}. 
    
    Though the classical MAB problem is important, practical applications generally require different variants. For instance, online shopping filters provide some contexts in the form of filters to improve the quality of the search results. Here, Contextual MAB (\citet{Li2010}) are more suited than the classical MAB. In online advertising, the platform needs to shuffle through ads, and certain ads may not be available at certain times of the day. Researchers model this problem as a Sleeping MAB (\citet{Kleinberg2010}).
    
    This paper considers an essential variant, namely \emph{Combinatorial Bandits with Knapsack constraint} (CBwK). In this variant, the algorithm also incurs a cost from pulling an arm apart from realizing a stochastic reward from an unknown distribution. CBwK variant combines the two well-known variants in MABs, namely, Combinatorial MAB (CMAB) (\citet{Chen2013}) where a subset of arms need to be pulled at each round and Bandits with Knapsack (BwK) (\citet{Badanidiyuru2018}) where a single-arm needs to be pulled at each round but within a budget constraint. An algorithm for CBwK aims to pull a subset of arms at each round to maximize the reward and ensure the budget constraint. There are many applications where combinatorial bandits with knapsack constraints are required. E.g., crowdsourcing where a subset of workers are required to be selected to do a given set of tasks, and there is a fixed budget to do so; internet routing where a path consisting of many nodes needs to be selected to send the packet within the time budget; and Internet of Things where a set of sensors need to be selected at every instance and taking data from a sensor is costly. 
    
    For example, consider a crowdsourcing platform where requesters can post their tasks that need to be completed by a team of workers. Consider that the requester has a series of tasks that need to be performed and has a fixed budget at disposal to hire the workers. The requester would like to hire the best workers to complete each task, but each worker is also associated with a cost. Further, the quality of workers for the task is unknown to the requester and has to be learnt over time. Note that the budget is a strict constraint, and no workers can be selected if the budget is exhausted, even if some tasks are left. In addition to the classical exploration-exploitation dilemma, subset selection with budget constraints poses an additional challenge to the requester. That is, how many workers the requester should select at each round. If it selects too many workers in the initial rounds, learning can be made faster. However, this would also result in incurring too much cost and hence exhaustion of the budget rapidly. On the other hand -- if it selects only a few workers at each round, the tasks may be completed with inferior quality. Much budget may be left at the end, which ideally the requester could have used in hiring more workers for the tasks. Works like \cite{Jain2016, Zheyuan2020} have used MABs in crowdsourcing setting.
   
    Even though there has been extensive research on both the variants CMAB (\cite{Xu2020, Chen2018}) and Budgeted MAB (BMAB) (\cite{Xia2017, Trov2016}), not much has been explored in Budgeted Combinatorial MABs (BCMAB). There is recent work on BCMAB by \citet{Karthik2017}. However, the algorithm divides the budget uniformly over all rounds, which leads to regret of O($\sqrt{T}\log{T}$). Regret of a given algorithm is the difference between the expected reward of an optimal algorithm that has access to the unknown parameters and the expected reward of the given algorithm.
    
    In our setting, the challenge is that no super-arm is optimal across all rounds, whereas in approaches where budget is equally divided in all rounds, there is an optimal super-arm that the algorithm tries to identify. It shows the complexity of designing an algorithm and its regret analysis. We first propose a greedy algorithm for \emph{offline setting} where we assume knowledge of the mean rewards. With this as offline solutions, we propose \greedy\ an UCB-based approach for CBwK. One can easily observe that the regret for a fixed number of rounds drops to zero if we have enough budget to explore all the arms in every round. Though \greedy\ performs well, the regret analysis is elusive due to possibly selecting different super-arms every round. To resolve the regret challenges, we devise another algorithm \ourbwk\ by leveraging \texttt{PrimalDualBwK} from \cite{Badanidiyuru2018} using an ingenious reduction to their settings. We provide regret guarantees for \ourbwk. We believe that \greedy\ should have better performance, but we do not have formal proof. It is interesting to note that for fixed budget, the regret of \ourbwk\ is $O(\log^2{T})$ as compared to $O(\sqrt{T}\log{T})$ of \citet{Karthik2017}. In summary, the following are our contributions. 
    
\subsection{Contribution}
\label{ssec:contribution}

    To the best of our knowledge, we are the first to address CBwK that does not restrict the number of arms for any round while working with an overall budget instead of fixed budget per round. In particular, 
    \begin{enumerate}
        \item We propose two algorithms, namely, \greedy\  which greedily allocates the arms in each round, and \ourbwk\ inspired by existing PrimalDualBWK algorithm, which works for single pull in bandits with knapsack setting \citet{Badanidiyuru2018}.
        \item We rigorously prove that \ourbwk\ achieves a regret of $O(\sqrt{\log{(n^2T)}}(\sqrt{n \cdot OPT_{LP}} + OPT_{LP}\sqrt{\frac{n}{B'}}) +  n\log{(n^2T)}\log{T})$ (Theorem \ref{th:regret}) which is the best so far in the area of combinatorial bandits with knapsack where the number of arms is not fixed at each round. Here, $n$ is the number of arms, $T$ is the number of rounds, $B' = min(B, T)$ and $OPT_{LP}$ is the optimal value of the LP relaxation to our problem.
        \item We experimentally show that \greedy\  performs incrementally better than \ourbwk, and both outperform existing works in the literature.
    \end{enumerate}

\section{Related Work}
\label{sec:rel_work}
    
    \noindent \citet{LaiRobbins1985} initiated the work on stochastic MABs and showed that it is possible to achieve an asymptotic regret of O($\log T$), where $T$ is the number of rounds played. The UCB1 algorithm by \citet{Auer2002} achieved regret of O($\log T$) uniformly over time. Since then, a lot of work has been done on stochastic MABs and their variants using the ideas from UCB1 algorithm.
    
    
    Stochastic combinatorial multi-armed bandits (CMABs) were first introduced by \citet{Chen2013}. The authors used an $(\alpha, \beta)$-approximation oracle that takes the distributions of arms and outputs a subset of arms with probability $\beta$ generates an $\alpha$ fraction of the optimal expected reward. The concept of $(\alpha, \beta)$-approximation oracle has been used in many CMAB related works such as \citet{Chen2018}. This setting was extended to linear rewards by \citet{Wen2015}. Recently, the regret of Thompson sampling-based algorithm was derived for CMAB by \citet{wang2018thompson}.
    All the above works did not work with any constraints. A few works in the literature have considered combinatorial bandits setting with constraints to be satisfied in each round. For example, \citet{Jain2018} considers a setting where at each round the subset of arms need to be selected to satisfy a quality constraint. Other examples of settings that have considered combinatorial bandits include \cite{Bhat2015, Sarma2010, Deva2021, JainGujar2020}.
    
   
    Budgeted multi-armed bandits (BMABs) were introduced in \citet{Tran2010}, and the concept of Knapsack in BMABs was later elaborated on in \citet{Tran2012}. Works like \citet{Nhat2019}, \citet{Tran2013}, \citet{Hao2015}, \citet{Jennings2012}, \citet{Jain2018}, \citet{Singh2021} deal with variants of BMAB, but none of them consider the combinatorial setting. \citet{Badanidiyuru2018} provide a setting of Bandits with Knapsacks as a generalization of several models. Our work is closely related to this, but they also consider a single pull setting as compared to our combinatorial setting.

    Even though CMABs and BMABs have been explored quite intensively, not much literature exists for both of them together. \citet{Chen2018}, \citet{Zhou2017}, \citet{Gao2020} provide a setting that assumes a fixed size of super-arm. The closest work to ours is by \citet{Karthik2017}. They consider a similar setting of BCMAB with an overall budget and no restriction on the size of super-arm. However, they distribute the budget uniformly over all rounds and incurs a regret of O($\sqrt{T}\log{T}$), where $T$ is the number of rounds. Uniformly distributing the overall budget can result in arbitrary worst regret when the restriction is lifted. We do not impose this restriction of consuming uniform budget over all rounds, and our \ourbwk\ algorithm incurs a regret of O($\log^2 T$), which is a substantial improvement.
    

\section{Preliminaries}
\label{sec:preliminaries}

    Let $\mathcal{N}=\{1,2,\ldots,n\}$ denote the set of arms, where each arm $i$ has a stochastic reward with fixed but unknown distribution with unknown mean $\mu_i \in [0,1]$. We represent the vector of mean rewards as $\mu = \{ \mu_1, \mu_2, \ldots, \mu_n \}$. Each arm $i$ when pulled further incurs a known cost, $c_i \in [0,1]$. There is a total of $T$ rounds available along with a total budget of $B$. The algorithm is allowed to pull any number of arms in any round $t$. For example, in a crowdsourcing setting, these arms could be workers with $\mu_i$  being quality of worker $i$ and $T$ being the total number of tasks available. We consider the additive setting where the reward obtained by pulling the subset of arms is additive. If the algorithm pulls a super-arm $S_t\subset \mathcal{N}$ in round $t$, its expected reward is $\sum_{i\in S_t} \mu_i$. The algorithms observes reward from each arm -- i.e., it obtains \emph{semi-bandit feedback}. Let $N_i(t)$ be the number of times arm $i$ has been pulled till round $t$. 

    The goal of an algorithm $ALG$ is to maximize the total expected reward obtained in overall rounds within the budget. For every round $t$, a super-arm $S_t$ comprising of one or more arms is selected. The super-arm $S_t$ receives an expected reward $REW_{ALG}(S_t)$. The problem can be formulated as an Integer Programming as follows:            
    
    \begin{equation}
    \begin{split}
    \label{eq:lp}
        \max_{\{S_t\}_{t=1}^T} REW_{ALG} &\quad = {\sum_{t=1}^{T} \sum_{i \in S_t} \mu_i } \\
        \text{subject to } &\quad \sum_{i=1}^{n} N_i(T) c_i \le B \\
        &\quad 0 \le N_i(T) \le T
    \end{split}
    \end{equation}
    
    
    
    \noindent Optimization Problem in Equation \ref{eq:lp} is hard even when the mean rewards $\mu_i$s are known. For $T = 1$, the problem reduces to a Knapsack problem. In the unknown reward setting, \citet{Karthik2017} assumes fixed budget per round, which approximates Optimization Problem \ref{eq:lp} as $T$ independent Knapsack problems. However, this may perform arbitrarily bad, as we prove below.

    \begin{Lemma}
    \label{lemma:fix-budget}
        Let $T > 1$ and $ALG1$ be an algorithm that fixes the budget per round as $B' < B$. Further, let $ALG^*$ be the optimal algorithm. Then, the ratio $\frac{REW_{ALG^*}}{REW_{ALG1}}$ can be arbitrarily bad. 
    \end{Lemma}
    \begin{proof}
        Take $\mathcal{N} = \{1, 2\}$ and the costs of the arms be $c_1 = B' + \epsilon_1$ and $c_2 = \epsilon_2$, for some arbitrarily small positive $\epsilon_1 $ and $\epsilon_2$. Since budget per round is $B'$, $ALG1$ can never select arm 1 even if arm 1 has a much higher mean reward than arm 2. If the optimal algorithm $ALG^*$ selects arm 1 at least once (even if it selects no other arm), the ratio of expected rewards $\frac{REW_{ALG^*}}{REW_{ALG1}}$  be at least $\frac{\mu_1}{T\mu_2}$. As the ratio of $\frac{\mu_1}{\mu_2}$ can be arbitrarily large (for $\mu_1$ close to 1 and $\mu_2$ close to 0), fixing the budget can perform arbitrarily bad.
    \end{proof}

\begin{Proposition}
\label{prop:fix_arms}
Let $T > 1$ and $ALG1$ be an algorithm that fixes the number of arms to be pulled upfront, independent of the problem instance. Further, let $ALG^*$ be the optimal algorithm. Then, the ratio $\frac{REW_{ALG^*}}{REW_{ALG1}}$ can be arbitrarily bad.
\end{Proposition}
\begin{proof}
    (\emph{Proofsketch})
    Suppose, an algorithm decides to pull $K=2$ arms every round. Take $\mathcal{N} = \{1, 2\}$ and the costs of the arms be $c_1 = \epsilon$ and $c_2 = 1$ and mean values be such that $\mu_1 >\mu_2$. Let $B=T \epsilon$.
    Then $REW_{ALG1} = (\mu_1+\mu_2)\frac{B}{(c_1+c_2)}$ where as pulling arm 1 for every round gives reward of $REW_{ALG^*} = \frac{B}{c_1}\mu_1$. Thus, $\frac{REW_{ALG^*}}{REW_{ALG1}} = \frac{\mu_1(c_1+c_2)}{c_1(\mu_1+\mu_2)} > \frac{1}{2}(1+\frac{1}{\epsilon})$. 
    As $\epsilon$ can be arbitrarily small, this ratio can be arbitrarily bad.
\end{proof}
    Lemma \ref{lemma:fix-budget} and Proposition \ref{prop:fix_arms} illustrate that fixing budget or number of pulls at any round can lead to arbitrary bad reward with respect to optimal, and hence show the complexity of solving Optimization Problem \ref{eq:lp}.
    Since the problem is hard even for known reward setting, we  first discuss the algorithms that can be used for known rewards setting. 
    
    
    \emph{Note:} Optimization Problem \ref{eq:lp} can be modelled as a Dynamic Programming (DP) problem by multiplying the costs $c_i$ by some scaling factor $s$ to integral values (accordingly the budget $B$ now becomes $sB$). Even though we will get an optimal solution from this, it will run in O($nBTs^2$), which is fairly slow. Furthermore, in the unknown stochastic setting, this DP will have to be called after every round, making it infeasible for practical use.

Inspired by the greedy solution of knapsack problem, we propose a greedy approach in the next subsection which select the arms with respect to bangperbuck ratio.    
    
\subsection{Deterministic Setting with Known $\mu$ - Greedy Approach (\offgreedy)}
\label{ssec:known-setting-greedy}
    
    \noindent Since modelling Optimization Problem \ref{eq:lp} as a DP can be very slow, we make attempts towards a greedy approximation. When the rewards are given, we can reduce this problem to knapsack problem where each arm have $T$ copies and the goal is to select the a subset of these $nT$ arms available so as to maximize the reward subject to budget constraint. Note that this approach will not work in an online setting because, the decision has to be made at every time instance and an arm can be pulled atmost once at that time instance. We will see later how we can adapt the offline greedy algorithm to the online setting.
    
    Let us first define the bangperbuck ratio of an arm. The \emph{bangperbuck} ratio of an arm $i$ is given by the ratio of its value and cost, i.e., $\frac{\mu_i}{c_i}$.
    The core idea behind the greedy solution is to select arms with higher bangperbuck ratio maximum number of times possible before the arms with lower bangperbuck. The approach is as follows.
    Select the arm with the highest bangperbuck ratio for as many tasks as possible, without violating budget constraints. Then select the arm with the next highest bangperbuck ratio for as many tasks as possible, without violating budget constraints. Continue this until we are done with the arm with the lowest bangperbuck ratio. We present it formally in Algorithm \ref{algo:known-greedy}.
    
    \begin{algorithm}[!h]
    \caption{\offgreedy}
    \label{algo:known-greedy}
        \begin{algorithmic}[1]
        
        \STATE \textbf{Input}: Number of arms $n$, budget $B$, number of rounds $T$ and the mean values of the arms $\mu = \{ \mu_1, \ldots \mu_n \}$. \\ \textbf{Output}: Allocation $\{N_1(T), N_2(T), \ldots, N_n(T)\}$
        \STATE Initialize budget remaining $B_r = B, S_1 = S_2 = \ldots, S_T = \phi$
        \STATE Sort the arms in non-increasing order of their bangperbuck ratios. Number them as $x_1, x_2, \ldots, x_n$. Hence $\frac{\mu_{x_i}}{c_{x_i}} \ge \frac{\mu_{x_{i'}}}{c_{x_{i'}}} \forall$ pairs $i < i'$ .
        
        \FOR{$i = 1, 2, \ldots, n$}
            \STATE $N_{x_i}(T) = \min\{T, \left\lfloor\frac{B_r}{c_{x_i}}\rfloor\right\}$ 
            \STATE $B_r = B_r - c_{x_i}N_{x_i}(T)$
        \ENDFOR
    
        \end{algorithmic}
    \end{algorithm}
We have the following remarks with respect to \offgreedy\ algorithm:\\

\emph{Remark 1:} Since \offgreedy\ is similar to that of greedy algorithm of knapsack algorithm, it can be easily modified to achieve $\frac{REW_{ALG^*}}{REW_{\offgreedy}} = 2$\\

\emph{Remark 2:} \offgreedy\ runs in O($n\log n$), dominated by the sorting step. This is much faster than the DP solution.

\subsection{Deterministic Setting with Known $\mu$ - Modelling it as LP (\CBwK)}
\label{ssec:lp}

We can also write the optimization problem in Equation \ref{eq:lp} as a linear programming problem similar to \cite{Badanidiyuru2018}. Following their works, we propose the following reduction to their single pull setting:
    
The new setting has $T$ rounds with every round consisting of $n$ plays: on the $i^{th}$ play, the algorithm can choose whether to pull arm $i$, or to not pull any arm. The setting is thus transformed into a single pull setting with $nT$ rounds, where it is not required to pull an arm in every round. To incorporate this reduction, we consider $n$ additional resources. We model each arm as a budget constrained resource, such that each arm can be pulled at most $T$ times, and deterministically consumes one unit per pull. Therefore, now instead of just having a cost $c_i$ associated with each arm $i$, there is a cost vector $C_i$ associated with each arm $i$. The length of $C_i$ vector will be $n+1$, with first $n$ components corresponding to $n$ additional resource and last one denoting the cost $c_i$. Note that for each arm $i$, the $i^{th}$ component will be one, the last component will be $c_i$ and rest all will be zero.  
    
Now, let $B' = \min (B, T)$. We scale the costs $c_i$ for each arm $i$ as well as the costs for the additional resources to make all budgets uniformly equal to $B'$. Thus, our cost matrix $M$, which is of the size $(n+1) \times n$, can be written as: 
    \begin{equation}
    \label{eq:C-matrix}
        M_{ji} =
        \begin{cases}
            c_i \cdot B'/B & \text{if}\ j = N+1 \\
            1 \cdot B'/T, & \text{if}\ i=j \\
            0, & \text{otherwise}
        \end{cases}
    \end{equation}
Here, $M_{ji}$ indicates cost of resource $j$ if we pull arm $i$. Now, we write this as a relaxed LP as shown in Equation \ref{eq:lp-primal}. We also write its dual in Equation \ref{eq:lp-dual}.

    \begin{equation}
    \begin{split}
    \label{eq:lp-primal}
        &\text{max} \sum_{i \in \mathcal{N}} \zeta_i \mu_i,\;\;\; \; \zeta_i \in \mathbb{R},\\
        &\text{s.t.} \sum_{i \in \mathcal{N}} \zeta_i M_{ji} \le B'\;\;\;\; \forall j\in\{1,2,\ldots, n+1\}\\
         &\zeta_i \ge 0,\;\;\;\; \forall i\in \mathcal{N}
    \end{split}
    \end{equation}
    \noindent The variables $\zeta_i$ represent the fractional relaxation for the number of rounds in which a given arm $i$ is selected. This is a bounded LP, because $\sum_{i \in \mathcal{N}} \zeta_i \mu_i \le \sum_{i \in \mathcal{N}} \zeta_i \le NT$. Let, the optimal value of this LP is denoted by $OPT_{LP}$. We now present the dual formulation of the problem.
    \begin{equation}
    \begin{split}
    \label{eq:lp-dual}
        &\min B' \sum_j \eta_j, \;\;\;\; \eta_j \in \mathbb{R}\\
        &\text{s.t.}  \sum_j \eta_j M_{ji} \ge \mu_i,\;\;\;\; \forall i \in \mathcal{N}\\
        & \eta_j \ge 0\;\;\;\;\forall j\in\{1,2,\ldots,n+1\} 
    \end{split}
    \end{equation}
    \noindent The dual variables $\eta_j$ can be interpreted as a unit cost for the corresponding resource $j$. We refer to the algorithm to solve above dual as \CBwK\ and it is easy to see $REW_{\CBwK} = OPT_{LP}$ (due to LP duality).
    
    \emph{Remark:} In a deterministic setting, it is fairly obvious why $OPT_{LP} \ge REW_{ALG^*}$. However, it is not as trivial when reward from arm pulls are stochastic. We prove this in the following Lemma.
    
    \begin{Lemma}
    \label{lemma:upper-bound-lp}
    $OPT_{LP}$ is an upper bound on the value of the optimal reward: $OPT_{LP} \ge REW_{ALG^*}$, where $ALG^*$ is an optimal algorithm.
    \end{Lemma} 
    
    \begin{proof} Let $\eta^* = (\eta^*_1, \ldots, \eta^*_d)$ denote an optimal solution to Equation \ref{eq:lp-dual}. Interpret each $\eta^*_j$ as a unit cost for the corresponding resource $j$. By strong LP duality, we have $B' \sum_j \eta^*_j = OPT_{LP}$. Dual feasibility implies that for each arm $i$, the expected cost of resources consumed when $i$ is pulled exceeds the expected reward produced. Thus, if we let $Z_t$ denote the sum of rewards gained in rounds $1, \ldots, t$ of the optimal dynamic policy, plus the cost of the remaining resource endowment after round $t$, then the stochastic process $Z_0, Z_1, \ldots, Z_T'$ is a supermartingale. Note that $Z_0 = B' \sum_j \eta^*_j = OPT_{LP}$, and $Z_{T'-1}$ equals the algorithm’s total payoff, plus the cost of the remaining (non-negative) resource supply at the start of round $T'$. By Doob’s optional stopping theorem, $Z_0 \ge E[Z_{T'-1}]$ and the lemma is proved.
    \end{proof}

\subsection{Regret}
\label{ssec:regret}

    
    Let the optimal algorithm for Optimization Problem \ref{eq:lp} be $ALG^*$. Its reward is given by $REW_{ALG^*}$. We define regret incurred by an algorithm $ALG$ as
    \begin{equation}
    \label{eq:regret-def}
        REG_{ALG}(\mathcal{N},T,B) = REW_{ALG^*}(\mathcal{N},T,B) - REW_{ALG}(\mathcal{N},T,B)
    \end{equation}


    
\section{Proposed Approaches}
\label{sec:prop_approach}

\subsection{\greedy}
\label{ssec:greedy-algo}
        
    \noindent Our first algorithm \greedy\ in unknown reward setting is presented in Algorithm \ref{algo:greedy}. The algorithm basically extends the greedy algorithm presented in section \ref{ssec:known-setting-greedy} so as to select a subset of arms at each round $t$.

    \begin{algorithm}[!h]
    \caption{\greedy}
    \label{algo:greedy}
        \begin{algorithmic}[1]
        
        \STATE \textbf{Input}: Number of arms $n$, budget $B$, number of rounds $T$ \\  \textbf{Output}: Subset of arms $S_t$ to be pulled for each round $t$.  
        \STATE Initialize remaining budget $B_r = B$
        \FOR{$i$ in $\mathcal{N}$}
            \IF{$c_i \le B_r$}
                \STATE Pull arm $i$ and update UCB value of arm $i$, $N_{i}(t) = 1$, and $B_r = B_r - c_i$
            \ENDIF
        \ENDFOR
        
        \FOR{rounds $t = 2, 3, \ldots, T$}
            \STATE Call $\offgreedy(n, B_r, T-t, UCB_1, UCB_2, \ldots, UCB_n)$ to get the allocation $N_1(T-t), N_2(T-t), \ldots, N_n(T-t)$.  
            \FOR{arm $i$ in $\mathcal{N}$}
                \IF{$N_i(T-t) \ge 1$ and $c_i \le B_r$}
                    \STATE $S_t = S_t \cup \{i\}$, $B_r = B_r - c_i$
                \ENDIF
            \ENDFOR
            \STATE Pull all the arms in $S_t$ and update UCB values of the arms $i$ and $N_{i}(t) = N_i(t) + 1\ \forall i \in S_t$. 
        \ENDFOR
    
        \end{algorithmic}
    \end{algorithm}
    
    \noindent Steps 3-7 correspond to the first round. In the first round, we select every arm in the super-arm as long as they can be pulled without exceeding the budget. Steps 8-15 correspond to the remaining $T-1$ round. At the start of each round, we use the greedy approach discussed in section \ref{ssec:known-setting-greedy} on the remaining budget and the remaining number of tasks to get the number of times each arm should be selected, if the UCB estimates of each arm were their real mean values. Any arm that the greedy approach does not pull in the remaining tasks is not pulled by \greedy\ for the next round. All the remaining arms are pulled as long as they can be pulled without exceeding the budget. The idea here is that if the greedy algorithm does not select an arm even once, it either has very low bangperbuck ratio (as per the UCB estimate until that round), or has a higher cost than the remaining budget. In either case, it makes sense to not select the arm for that round.

    \noindent \emph{Note:} We do not prove a regret bound for \greedy\ as the analysis gets quite tricky since it is not easy to estimate the size of super-arm $S_t$ selected in any round $t$. 
    
\subsection{\ourbwk}
\label{ssec:ourbwk-algo}
        
    \noindent The intuitive idea for the \ourbwk\ algorithm is to greedily select arms with the highest estimated bangperbuck ratio. Here, the bangperbuck ratio for a given arm $i$ is defined as $\mu_i / (\eta^* C_i)$, where the denominator represents the expected cost of pulling this arm. The algorithm \ourbwk\ is formally stated in Algorithm \ref{algo:ourbwk}.
        
    \begin{algorithm}[!h]
    \caption{\ourbwk}
    \label{algo:ourbwk}
        \begin{algorithmic}[1]
            \STATE \textbf{Input}: Number of arms $n$, budget $B$, number of rounds $T$ \\ \textbf{Output}: Allocation $\{S_1, S_2, \ldots, S_T\}$
            \STATE $S_t = \phi\ \forall t$
            \STATE \textbf{Initialization}
            \STATE In the first round, pull arm $i$ on the $i^{th}$ play i.e. $S_1 = S_1 \cup \{i\}$.
            \STATE $v_1 = \textbf{1} \in [0, 1]^{n+1}$
            \STATE Set $\epsilon = \sqrt{\frac{\ln{(n+1)}}{\min{(B, T)}}}$
            
            \FOR{rounds $t = 2, 3, \ldots , T$}
                \FOR{each arm $i \in \mathcal{N}$}
                    \STATE Compute UCB estimate for the expected reward, $u_{t,i} \in [0, 1]$
                    \STATE Expected cost for one pull of arm $i$ is estimated by $EstCost_i = C_i \cdot v_t$
                \ENDFOR
                \FOR{play $i = 1, 2, \ldots, n$}
                    \STATE $S_t = S_t \cup \{i\}$ if there is enough budget remaining to pull arm $i$ after all the estimated better arms (arms with higher estimated bangperbuck) have been pulled for the remaining rounds 
                    \STATE If pulled, update estimated unit cost for each resource $j$:
                    $v_{t+1}(j) = v_t(j) (1+\epsilon)^{M_{ji}}$
                \ENDFOR
            \ENDFOR
        \end{algorithmic}
    \end{algorithm}
    
    \noindent Steps 1 and 2 corresponds to the first round where we pull the $i^{th}$ arm in the $i^{th}$ play. Steps 3 and 4 set the parameters. Steps 5-14 correspond to the remaining ($T-1$) rounds. In steps 6-9, for each arm, we update its UCB estimate and expected cost of pulling that arm. In steps 10-14 we decide whether to pull an arm or not by prioritizing arms with higher bangperbuck ratios. If it gets pulled, we update its estimated cost accordingly. $v_t$ gives an estimate of $\eta^*$ after set of rounds $t$.
    
  \ourbwk\ is inspired from PrimalDualBwK algorithm for single pull setting. Here, we list down the key differences among the two algorithms.
\subsubsection*{Key Differences between \ourbwk\ and PrimalDualBwK}
\label{sssec:differences}

\begin{itemize}[noitemsep, leftmargin=*]
    \item     The fundamental difference between \ourbwk\ and PrimalDualBwK is that every round in \ourbwk\ has $n$ plays whereas a round in PrimalDualBwK has only one play. \ourbwk\ can also choose to smartly not pull any arm in a play even if it is possible to, but PrimalDualBwK always pulls an arm every round as long as budget permits.
    \item     The main challenge in adapting PrimalDualBwK for combinatorial setting, and why it is not a trivial extension with number of rounds as $nT$, is as follows: Once an arm $i$ is selected by PrimalDualBwK in round $t$, it will once again be available for selection in round $t+1$. However, since several consecutive plays fall in the same round for \ourbwk, if arm $i$ is selected in round $t$ for play $p$, it will not be available for any other play $p'$ in round $t$. We tackle this by smartly choosing to not pull any arm in certain plays. An arm is considered for pull exactly once per round, and will only be selected if there is enough budget remaining to pull it after all the estimated better arms (arms with higher estimated bangperbuck) have been pulled for the remaining rounds.
    \item     Another important distinction is that, in \ourbwk, the UCB values of all the arms are updated together at the start of every round, and they do not change with every play.
\end{itemize}

    Note that, our main contribution is the reduction of our combinatorial setting into the single pull setting provided by \citet{Badanidiyuru2018} such that the generated single pull solution works in original setting flawlessly. \ourbwk\ is a modified version of their PrimalDualBwK algorithm. With such ingenious mapping, the regret proof becomes similar to that of \cite{Badanidiyuru2018}. We use their techniques to suits our combinatorial setting and present the regret analysis for \ourbwk and prove a regret bound of the form:
    \begin{equation}
    \label{eq:regret-bound1}
        OPT_{LP} - REW_{\ourbwk} \le f(OPT_{LP}) 
    \end{equation}
    \noindent where $f(\cdot)$ is a linear function that depends only on parameters $(B', n, T')$. Regret bound (\ref{eq:regret-bound1}) implies the claimed regret bounds relative to $REW_{ALG^*}$ because
    \begin{equation}
    \label{eq:regret-bound2}
        REW_{\ourbwk} \ge OPT_{LP} - f(OPT_{LP}) \ge REW_{ALG^*} - f(REW_{ALG^*})
    \end{equation}
    \noindent where the second inequality follows trivially because $g(x) = max(x-f(x), 0)$ is a non-decreasing function of $x$ for $x \ge 0$ for a linear $f(\cdot)$, and $OPT_{LP} \ge REW_{ALG^*}$ from Lemma \ref{lemma:upper-bound-lp}. From Equation \ref{eq:regret-bound2} and Equation \ref{eq:regret-def},
    \begin{equation}
    \label{eq:regret-bwk}
        REG_{\ourbwk} = REW_{ALG^*} - REW_{\ourbwk} \le f(REW_{ALG^*})
    \end{equation}
    
\subsection{Differences with SemiBwK-RRS}
\label{ssec:diff-sbwk}

    Even though \citet{Karthik2017} also consider a budgeted combinatorial setting with semi-bandit feedback, the key difference lies in the fact that they assume a fixed budget per round. They solve an LP in each round which considers a pre-determined fixed budget for every round. We have already shown in Lemma \ref{lemma:fix-budget} that this can perform arbitrarily bad. We use budget in each round adaptively, thus leading to better regret bounds.

    \noindent In the next section, we bound our regret and prove Equation \ref{eq:regret-bwk}.

\section{Regret Analysis of \ourbwk}
\label{sec:regret-analysis}

    We start this section by introducing some notations that will be used in our regret analysis. Then we formally present our result in Theorem \ref{th:regret}.
    For the sake of clarity in notations, let $d = n+1$ be the dimension of the cost vector (represented as $C_i$ for arm $i$). Take distribution $y_t$ as a vector of normalized costs of resources, i.e., $y_t(j) = \frac{v_t(j)}{\sum_{j=1}^dv_t(j)}$.  Let us take $W$ as the total expected normalized cost consumed by the algorithm after the first round (hence, $W = \sum_{t=2}^{T} \sum_{i \in S_t} y_t^T C_i$).  $S_t:$  the set of arms selected in their respective pulls in round $t$. With this, we claim the regret guarantee for \ourbwk:
    
    \begin{Theorem}
    \label{th:regret}
        The regret of algorithm \ourbwk\ with parameter $\epsilon = \sqrt{\frac{\ln{d}}{B'}}$, for $d = n+1$, satisfies
            \begin{align*}
            &OPT_{LP} - REW_{\ourbwk} \le \\ &O(\sqrt{\log{(n^2T)}})(\sqrt{n \cdot OPT_{LP}} + OPT_{LP}\sqrt{\frac{n}{B'}}) + O(n)\log{(n^2T)}\log{(T)}
            \end{align*}                
    \end{Theorem}
\begin{proof}(Overview) Using useful result adapted from \citet{Kleinberg2007}, 
     in Section \ref{ssec:regret-deterministic}, we bound the term $OPT_{LP} - f(OPT_{LP})$ for the deterministic (but unknown) setting. In Section \ref{ssec:mod-error}, we extend this for the unknown stochastic setting, as described in Equation \ref{eq:regret-bwk}, completing the proof.
         
\end{proof}    
    Step 12 in Algorithm \ref{algo:ourbwk} uses the multiplicative weights update technique by \citet{Freund1997}. It is an online technique for maintaining a $d$-dimensional probability vector $y$ while observing a sequence of $d$-dimensional payoff vectors $\pi_1, \ldots, \pi_\tau$. We use the following related result adapted from \citet{Kleinberg2007}.
    
    \begin{Proposition}
    \label{prop:hedge}
        Fix any parameter $\epsilon \in (0, 1)$ and any stopping time $\tau$ . For any sequence of payoff vectors $\pi_1, \ldots, \pi_\tau \in [0, 1]^d$, we have
        $$
        \forall y \in \Delta[d] \quad \quad \sum_{t=1}^\tau y_t^T \pi_t \ge (1-\epsilon) \sum_{t=1}^\tau y^T \pi_t - \frac{\ln{d}}{\epsilon}
        $$
    \end{Proposition}

\subsection{Deterministic Rewards}
\label{ssec:regret-deterministic}

    In this subsection, we consider the setting where pulling an arm $i$ deterministically generates reward $\mu_i$. We bound the regret incurred if the arms are pulled as per the algorithm \ourbwk.
    
    The payoff vector in any round $t > 1$, is given by $\pi_t = \sum_{i \in S_t} C_i$. Take the total cost consumed by Algorithm \ref{algo:ourbwk} as $W = \sum_{t=2}^T \sum_{i \in S_t} y_t^T \cdot C_i$. We want to maximise this $W$. 
    
    To see why $W$ is worth maximizing, let us relate it to the total reward collected by the algorithm in rounds $t > 1$ denoted by $REW_{\ourbwk} = \sum_{t=2}^T rew_t$, where $rew_t$ is the reward collected in the round $t$. We will prove in Lemma \ref{lemma:maximizeW} that $REW_{\ourbwk} \ge W \cdot \frac{OPT_{LP}}{B'}$. For this reason, maximizing $W$ also helps maximize $REW$.
    
    Let $\zeta^*$ denote an optimal solution of the primal linear program (LP-primal). Then $OPT_{LP} = \mu^T \zeta^*$ denote the optimal value of that LP.
    
    \begin{Claim}
    \label{claim:best-dist}
        We claim that there exists a $z_t$, such that
        \begin{equation}
        \label{eq:best-dist}
            z_t \in \arg\max_{z \in \Delta[\mathcal{N}]} \frac{\mu^Tz}{y^T_tMz}
        \end{equation}
    \end{Claim}
    
    \begin{proof}
    $z_t$ is a distribution that maximizes the bangperbuck ratio among all distributions $z$ over arms. Indeed, the argmax in Equation (\ref{eq:best-dist}) is well-defined as that of a continuous function on a compact set. Say it is attained by some distribution $z$ over arms, and let $\rho \in R$ be the corresponding max. By maximality of $\rho$, the linear inequality $\rho y^T_tMz \ge \mu^T$z also holds at some extremal point of the probability simplex $\Delta[\mathcal{N}]$, i.e. at some point-mass distribution. For any such point-mass distribution, the corresponding arm maximizes the bang-per-buck ratio in the algorithm. 
    \end{proof} 

    \begin{Lemma}
    \label{lemma:maximizeW}
        $REW_{\ourbwk} \ge W \cdot \frac{OPT_{LP}}{B'}$
    \end{Lemma}
    
    \begin{proof}
        \begin{align*}
        y^T_t \pi_t = y^T_t M z_t &\le \frac{rew_t (y_t^T M \zeta^*)}{OPT_{LP}}\\
            W &\le \frac{1}{OPT_{LP}} \sum_{t=2}^T rew_t (y^T_t M \zeta^*) \\ 
            &= \frac{1}{OPT_{LP}} \sum_{t=2}^T (rew_t y_t^T) M \zeta^*
        \end{align*}
        
        Now, let $\bar{y} = \frac{1}{REW_{\ourbwk}} \sum_{t=1}^T rew_t \cdot y_t \in [0, 1]^d$ be the rewards weighted average of distributions $y_{2} , \ldots, y_T$, it follows that
        
        \begin{align*}
            W &\le \left(\frac{REW_{\ourbwk}}{OPT_{LP}}\right) \bar{y}^T M \zeta^*\le \left(\frac{REW_\ourbwk}{OPT_{LP}}\right)B'
        \end{align*}
        \noindent The last inequality follows because all components of $M\zeta^*$ are at most $B'$ by the primal feasibility of $\zeta^*$.
        
        \noindent Hence $REW_{\ourbwk} \ge W \cdot \frac{OPT_{LP}}{B'}'$
    \end{proof}

    \noindent Combining Lemma \ref{lemma:maximizeW} and the regret bound from Proposition \ref{prop:hedge}, we obtain $\forall y \in \Delta[d]$
    \begin{align*}
    \label{rewhedge}
        REW_\ourbwk 
        \ge W \frac{OPT_{LP}}{B'} \ge \left[(1-\epsilon) \sum_{t=2}^T y^T M z_t - \frac{\ln{d}}{\epsilon}\right] \frac{OPT_{LP}}{B'}   
    \end{align*}
    
    \noindent The algorithm stops after round $T$. By this round, either $B$ (from the original setting in Optimization Problem \ref{eq:lp}) has been fully comsumed, or some arm has been selected $T$ times. Hence the consumption of some resource $j$ is at least $B'$. In a formula: $\sum_{t=1}^T y_t^T M z_t \ge B'$. Since the total cost for any resource is at max $n$ for $t=1$, $\sum_{t=2}^T y_t^T M z_t \ge B'- n$. Hence,
    \begin{eqnarray*}
        REW_{\ourbwk} & \ge [(1-\epsilon)(B'-n) - \frac{\ln{d}}{\epsilon}] \cdot \frac{OPT_{LP}}{B'} \\
            & \ge OPT_{LP} - [\epsilon B + n + \frac{\ln{d}}{\epsilon}] \cdot \frac{OPT_{LP}}{B'} \\
            & = OPT_{LP} - O(\sqrt{B\ln{d}} + n) \cdot \frac{OPT_{LP}}{B} \\ 
            & \text{where } \epsilon = \sqrt{\frac{\ln{d}}{B'}} \text{ and } d = n+1
    \end{eqnarray*}


\subsection{Stochastic Rewards}
\label{ssec:mod-error}

    Here, we use the techniques from the previous subsection and provide regret bound for the stochastic setting (Equation \ref{eq:regret-bwk}).
    
    The algorithm computes UCBs on expected rewards $u_{t,i} \in [0, 1]$, for each arm $i$ after every round $t$. The vector $u_t  \in [0, 1]^n$ represents these UCBs, where $i^{th}$ component equals $u_{t,i}$. Let $M$ be the resource-consumption matrix. That is, $M \in [0, 1]^{d \times n}$ denotes the matrix whose $(j, i)^{th}$ entry $M_{ji}$ is the actual consumption of resource $j$ if arm $i$ were chosen.
   
    As in the Section \ref{ssec:regret-deterministic}, we claim there exists a $z_t$ such that
    \begin{equation}
    \label{bestdist2}
        z_t \in \arg\max_{z \in \Delta\mathcal{N}} \frac{u_{t}^Tz}{y^T_tMz}
    \end{equation}
    \noindent As before, $z_t$ is a distribution that maximizes the bangperbuck ratio among all distributions $z$ over arms.
    
    We define a confidence radius $rad(x, N) = \sqrt{\frac{xC_{rad}}{N}} + \frac{C_{rad}}{N}$, where $C_{rad} = \theta(\log{ndT})$. We adapt the following result from \citet{Kleinberg2008} and \citet{Babaioff2015}.
    \begin{Proposition}
    \label{prop:confidence}
        Consider some distribution with values in [0, 1] and expectation $x$. Let $\hat{x}$ be the average of $N$ independent samples from this distribution. Then $\forall C_{rad} > 0$
        \begin{eqnarray*}
            Pr[|x - \bar{x}| \le rad(\bar{x}, N) \le 3rad(x, N)]
            \ge 1 - \exp{-\omega(C_{rad})}
        \end{eqnarray*}
    \end{Proposition}  
    
   Using Proposition \ref{prop:confidence} and our choice of $C_{rad}$, it holds with probability at least $1 - T^{-1}$ that the confidence interval for every latent parameter, in every round of execution, contains the true value of that latent parameter. We call this high-probability event a \emph{clean execution} of\ourbwk. Our regret guarantee will hold deterministically assuming that a clean execution takes place. The regret can be at most $T$ when a clean execution does not take place, and since this event has probability at most $T^{-1}$ it contributes only $O(1)$ to the regret. Now, we  assume a clean execution of \ourbwk.
   
    \begin{Claim}
    \label{claim:regret-stoch}
        In a clean execution of Algorithm PrimalDualBwK with parameter $\epsilon = \sqrt{\frac{\ln{d}}{B}}$, the algorithm’s total reward satisfies the bound
        \begin{align*}
            OPT_{LP} - REW_{\ourbwk} \le
            2 OPT_{LP}(\sqrt{\frac{\ln{d}}{B'}} + \frac{n}{B}) + n + | \sum_{t=2}^T \delta_t^T z_t |
        \end{align*}
        where $d = n+1$ and $\delta_t^T = u_t-\mu$ for each round $t$.
    \end{Claim}
    
    Let $\zeta^*$ denote an optimal solution of the primal linear program given in Equation \ref{eq:lp-primal}, and let $OPT_{LP} = \mu^T\zeta^*$ denote the optimal value of that LP. Let $REW_{UCB} = \sum_{t=2}^T u_t^T z_t$ denote the total payoff the algorithm would have obtained, after its initialization phase, if the actual payoff at time $t$ were replaced with the upper confidence bound.
    
    \noindent As before, $\sum_{t=1}^T y_t^T M z_t \ge B'$. Once again, since the total cost for any resource is at most $n$ for $t=1$, $\sum_{t=2}^T y_t^T M z_t \ge B'-n$
    
\noindent Let $\bar{y} = \frac{1}{REW_{UCB}} \sum_{t=2}^T (u^T_t z_t) y_t$. Assuming a clean execution,
    
\allowdisplaybreaks
    \begin{align*}
    \allowdisplaybreaks
        B' & \ge \bar{y}^T M \zeta^* = \frac{1}{REW_{UCB}} \sum_{t=2}^T (u^T_t z_t) (y_t M \zeta^*) \\
            & \ge \frac{1}{REW_{UCB}} \sum_{t=2}^T (u^T_t \zeta^*) (y_t M z_t) \\
            & \ge \frac{1}{REW_{UCB}} \sum_{t=2}^T (r^T \zeta^*) (y_t M z_t) \\
            & \ge \frac{OPT_{LP}}{REW_{UCB}} \left[ (1-\epsilon)y^T \left(\sum_{t=2}^T M z_t\right) - \frac{\ln{d}}{\epsilon}\right] \\
        \implies REW_{UCB} & \ge  \frac{OPT_{LP}}{B'} [B' - \epsilon B' - n - \frac{\ln{d}}{B}] \\
        \implies REW_{UCB} & \ge OPT_{LP} [1 - \epsilon  - \frac{n}{B} - \frac{\ln{d}}{\epsilon B}]
    \end{align*}
    
    \noindent The algorithm’s actual payoff, $REW_{\ourbwk} = \sum_{t=1}^T \mu^T z_t$, satisfies the inequality
    \begin{align*}
    REW_{\ourbwk} &\ge REW_{UCB} - \sum_{t=2}^T (u_t - \mu)^T z_t \\ & = REW_{UCB} -  \sum_{t=2}^T \delta^T_t z_t
    \end{align*}
    
    We use the following proposition from \citeauthor{Badanidiyuru2018}.
    
    \begin{Proposition}
    \label{prop:error-bound}(\citet{Badanidiyuru2018})
        Consider two sequences of vectors $a_1, \ldots, a_\tau$ and $b_1, \ldots, b_\tau$, in $[0, 1]^n$, and a vector $a_0 \in [0, 1]^n$. For each arm $i$ and each round $t > 1$, let $\bar{a}_{t,i} \in [0, 1]$ be the average observed outcome up to round $t$, i.e., the average outcome $a_{s,i}$ over all rounds $s \le t$ in which arm $i$ has been chosen by the algorithm; let $N_{t,i}$ be the number of such rounds. Assume that for each arm $i$ and all rounds $t$ with $1 < t < T$, we have
        $$
        |b_{t,i} - a_{0,i}| \le 2 rad(\bar{a}_{t,i},N_{t,i}) \le 6 rad(a_{0,i}, N_{t,i)},
        $$
        $$
        |\bar{a}_{t,i} - a_{0,i}| \le rad(\bar{a}_{t,i},N_{t,i}).
        $$
        Let $A = \sum_{t=1}^T a_{t, i}$ be the total outcome collected by the algorithm. Then,
        $$
        | \sum_{t=2}^T (b_t - a_t)^T z_t | \le O(\sqrt{C_{rad}nA} + C_{rad}n \log{T}).
        $$
    \end{Proposition}

    \noindent Taking $a_t = \mu$, $b_t = u_t$ and vector $a_0 = \mu$ in Proposition \ref{prop:error-bound}, we get
    \begin{equation}
    \label{eq:ucb-bound}
        | \sum_{t=2}^T \delta_t z_t| \le O(\sqrt{C_{rad} n REW_{\ourbwk}} + C_{rad} n \log{T}).
    \end{equation}
    
    \noindent We now combine the results to provide the proof for Theorem \ref{th:regret}.
    
\begin{proof}(Theorem \ref{th:regret})
        For $n \ge \frac{B'}{\log{dT}}$, the bound in Claim \ref{claim:regret-stoch} is trivially true. Therefore we can assume without loss of generality that $n < \frac{B'}{\log{dT}}$. We observe that 
        $$
        OPT_{LP} \left( \sqrt{\frac{\ln{d}}{B'}} + \frac{n}{B'}\right) = O\left(\sqrt{n \log{(ndT)}} \frac{OPT_{LP}}{\sqrt{B}}\right)
        $$
        The term $n$ on the right side of the bound in Claim \ref{claim:regret-stoch} is bounded above by $n\log{(dnT)}$.
        
        The theorem follows by plugging in $C_{rad} = \theta(\log{ndT})$, and $d = n+1 = \theta(n)$, in Equation \ref{eq:ucb-bound}, along with Claim \ref{claim:regret-stoch}. 
    \end{proof}

\section{Simulation-based Experiments}
\label{sec:exp}
In this section, we compare our proposed algorithms \greedy\ and \ourbwk\ with the existing SemiBwK-RRS algorithm for CBMAB problem. We begin with explaining the experimental setting and then analyse the results obtained. 
 \begin{figure*}[h]
            \begin{subfigure}{.33\textwidth}
        	\centering
    	    \includegraphics[scale=0.33]{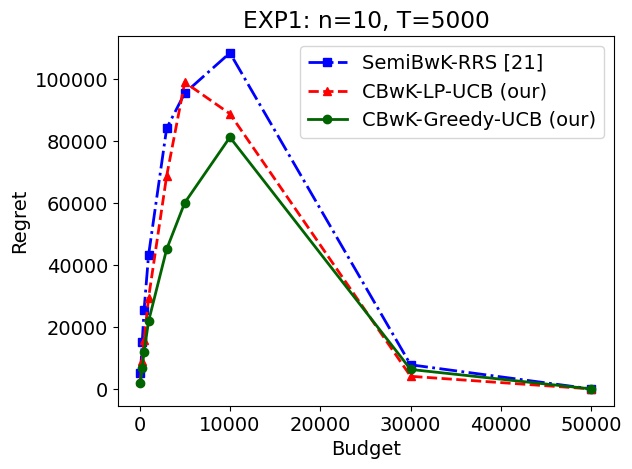}
    	    \caption{wrt $B$, for fixed $T$}
            \label{fig:regretB}
            \end{subfigure}	
            \begin{subfigure}{.33\textwidth}
        	\centering
    	    \includegraphics[scale=0.33]{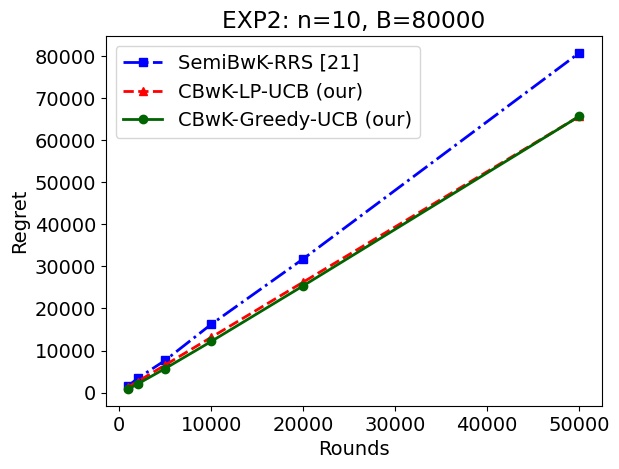}
    	    \caption{wrt $T$, for fixed $B$}
            \label{fig:regretT}
            \end{subfigure}	
            \begin{subfigure}{.33\textwidth}
        	\centering
    	    \includegraphics[scale=0.33]{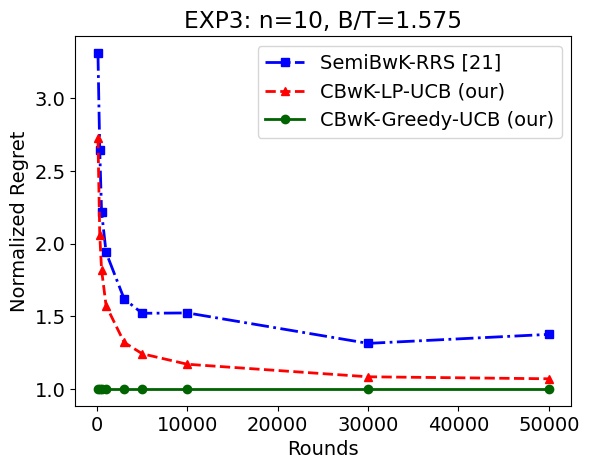}
    	    \caption{wrt $T$, for fixed $\frac{B}{T}$}
            \label{fig:regretBT}
            \end{subfigure}	
        \caption{Regret Comparison under different settings when the rewards and costs of the arm are drawn independently from identical distribution}
        \label{fig:regretvary}
        \end{figure*}
    
\subsection{Experimental Set-up}
\label{ssec:setup}
    For simulation of arms, we generate mean values and costs as follows: $\forall i \in \mathcal{N}, \mu_i \sim U[0, 1], c_i \sim U[0, 1]$. We take $\alpha = 5$ for \greedy as well as SemiBwK-RRS. We take $\epsilon = 0$ in SemiBwK-RRS, to maintain consistency with what \citet{Karthik2017} mentioned in their experiments. We report average over 100 randomly generated instances for each of the following experiments.
    
    \noindent\textsf{EXP1:} Varying $B=100\rightarrow 50000$ for $n=10$ and $T=5000$ to study effect of the budget on regret.
    
    \noindent\textsf{EXP2:} We vary $T=1000\rightarrow 50000$ for fixed budget $B=80000$ and $n=10$ to study effect of increased rounds on regret. 
      
    \noindent\textsf{EXP3:} We vary $T=1000\rightarrow 50000$ for fixed budget/round ratio $B/T=1.575, n=10$ to study effect of the increased task (with proportional budget increase) on regret.
     
    \noindent\textsf{EXP4:} For non i.i.d. arms (Figure \ref{fig:non-iid}), we vary $T=100\rightarrow 2000$ for fixed budget/round ratio $B/T=1.575$ and $n=4$ to study effect of the increased task (with proportional budget increase) on regret when the arms rewards are selected as follows: $\mu_1,c_1\sim U[0.9,1],\mu_2,c_2\sim U[0.6,0.8], \mu_3,c_3 \sim U[0.2,0.4], \mbox{ and } \mu_4,c_4 \sim U[0,0.1]$. In practical scenarios these arms can be seen as high, medium, low, very low rewarding arms.
\begin{figure}[h]
        	\centering
    	    \includegraphics[scale=0.33]{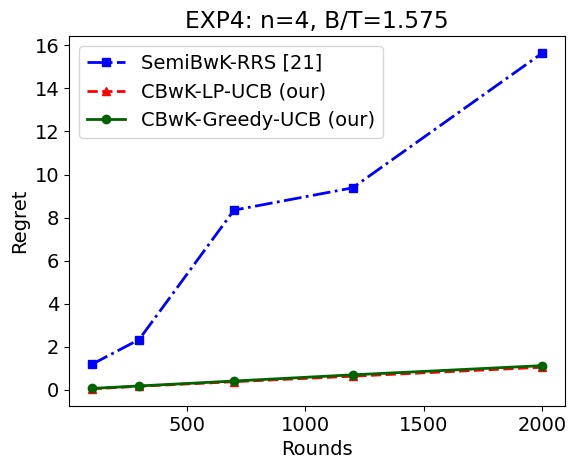}
    	    \caption{Regret w.r.t. $T$ for fixed $B/T$, non-i.i.d. arms}
    	    \label{fig:non-iid}
\end{figure}
\subsection{Empirical Analysis}
\label{ssec:comparison}
As can be seen from the figures that \greedy\ achieves the lowest regret in all four experiments. \ourbwk\ performs almost as good as \greedy, whereas SemiBwK-RRS performs a lot worse. The difference in the algorithms is less evident in EXP1 and EXP2 as regret dominantly depends  on $OPT_{LP}$, for all three algorithms and and $OPT_{LP}$ increases with increase in $B$ as well as $T$. Thus, to study relative performance of the algorithms as $T$ increase, we plot regrets relative to \greedy\ (\greedy\ normalized to 1) for EXP3. It clearly indicates superiority of our algorithms by not fixing the budget per round. We see that the regrets of \greedy\ and \ourbwk\ are 35\% better than that of SemiBwK-RRS in EXP1, EXP2 and EXP3, and multiple times better in EXP4. The standard deviation for \greedy\ was also not very high. For EXP1, EXP2 and EXP3, the Coefficient of Variance for \greedy\ was noted to be below 28\%, whereas it reached around 42\% in the worst case for SemiBwK-RRS.


\subsubsection*{Regret for High Budget}
\label{ssec:highbud}
    
    For budget values $B' >= T \sum_{i=1}^n c_i$, we get a regret of 0. It is because the optimal solution consists of selecting every arm for every round. In both \greedy\ as well as \ourbwk, we choose to not select an arm only when there is not enough budget to select it after the arms with higher bangperbuck have been selected. If the budget is enough to select all arms for all rounds, that scenario will never occur. Hence, the super-arms selected by \greedy, \ourbwk\ and optimal solution will be exactly the same in every round.

\section{Conclusion}
\label{sec:conclusion}
      We considered a Budgeted Combinatorial Multi-Armed Bandit setting with semi-bandit feedback. The existing literature has a fixed number of arm pulls (generally single pull) or a fixed pre-determined budget per round. We focused on the more general setting without any such restrictions. We first proposed \offgreedy, which uses a greedy technique to solve this in the offline setting. We showed why the problem is difficult to solve. We provided a reduction to the knapsack problem by creating $T$ copies of each arm and showed that our greedy algorithm is 2-approx. We designed \ourbwk\ using ingenious reduction to PrimalDualBwK by \citet{Badanidiyuru2018}. We provided regret bound for \ourbwk\ in the unknown stochastic setting. We compared our work with SemiBwK-RRS, the closest work to our setting, and experimentally demonstrated that \greedy\ and \ourbwk\  outperform SemiBwK-RRS.
    
    We believe an exciting direction to explore would be a non-additive setting such as a submodular combination of rewards. One can extend our work to variants such as sleeping bandits and contextual bandits. We believe our paper is the first step in this direction as the setting we have considered is relatively unexplored.
 
\newpage
\bibliographystyle{abbrvnat}
\bibliography{refer.bib}

\end{document}